\documentclass{article}

\usepackage{arxiv}

\usepackage[utf8]{inputenc} 
\usepackage[T1]{fontenc}    
\usepackage{hyperref}       
\usepackage{url}            
\usepackage{booktabs}       
\usepackage{amsfonts}       
\usepackage{nicefrac}       
\usepackage{microtype}      
\usepackage{lipsum}
\usepackage{amsthm}
\usepackage{amssymb}
\usepackage{amsmath,empheq, amsfonts}
\usepackage{url}
\usepackage{bm}
\usepackage{comment}
\usepackage{placeins}
\usepackage[ruled,vlined]{algorithm2e}

\usepackage{booktabs}       
\usepackage{subcaption}  
\usepackage{multirow}

\newtheorem{mylemma}{Lemma}

\DeclareMathOperator{\Tr}{Tr}
\DeclareMathOperator{\diag}{diag}

\newcommand*\diff{\mathop{}\!\mathrm{d}}

\title{The Bures Metric for Generative Adversarial Networks}

\author{Hannes De Meulemeester \\
	Department of Electrical Engineering\\
	ESAT-STADIUS, KU Leuven\\
	Kasteelpark Arenberg 10, B-3001 Leuven, Belgium\\
	\texttt{hannes.demeulemeester@kuleuven.be}  \And Joachim Schreurs \\
	Department of Electrical Engineering\\
	ESAT-STADIUS, KU Leuven\\
	Kasteelpark Arenberg 10, B-3001 Leuven, Belgium\\
	\texttt{joachim.schreurs@kuleuven.be}\And  Micha\"el Fanuel\thanks{Most of this work was done when Micha\"el Fanuel was at KU Leuven.} \\
	Univ. Lille, CNRS, Centrale Lille UMR 9189 CRIStAL \\
	F-59000 Lille, France \\ 
	\texttt{michael.fanuel@univ-lille.fr}\And Bart De Moor  \\
	Department of Electrical Engineering\\
	ESAT-STADIUS, KU Leuven\\
	Kasteelpark Arenberg 10, B-3001 Leuven, Belgium\\
	\texttt{bart.demoor@kuleuven.be}\And Johan A.K. Suykens \\
	Department of Electrical Engineering\\
	ESAT-STADIUS, KU Leuven\\
	Kasteelpark Arenberg 10, B-3001 Leuven, Belgium\\
	\texttt{johan.suykens@kuleuven.be}}

\begin{document}
	\maketitle
	
	\begin{abstract}
		Generative Adversarial Networks (GANs) are performant generative methods
		yielding high-quality samples. However, under certain circumstances, the training
		of GANs can lead to mode collapse or mode dropping.
		To address this problem, we use the last layer of the discriminator as a feature map to study the distribution of the real and the fake data. During training, we propose to match the
		real batch diversity to the fake batch diversity by using the Bures distance between
		covariance matrices in this feature space. The computation of the Bures distance
		can be conveniently done in either feature space or kernel space in terms of the covariance and kernel matrix respectively. We observe that diversity matching
		reduces mode collapse substantially and has a positive effect on sample
		quality. On the practical side, a very simple training procedure is proposed and assessed on several data sets.  
	\end{abstract}

	\section{Introduction}
	In several machine learning applications, data is assumed to be sampled from an implicit probability distribution. The estimation of this empirical implicit distribution is often intractable, especially in high dimensions. To tackle this issue, generative models are trained to provide an algorithmic procedure for sampling from this unknown distribution. Popular approaches are Variational Auto-Encoders proposed by~\cite{diederik2014auto}, Generating Flow models by~\cite{rezende2015variational} and Generative Adversarial Networks (GANs) initially developed by~\cite{Goodfellow}. The latter are particularly successful approaches to produce high quality samples, especially in the case of natural images, though their training is notoriously difficult. The vanilla GAN consists of two networks: a generator and a discriminator. The generator maps random noise, usually drawn from a multivariate normal, to fake data in input space.
	The discriminator estimates the likelihood ratio of the generator network to the data distribution. It often happens that a GAN generates samples only from a few of the many modes of the distribution. This phenomenon is called `mode collapse'.
	
	\paragraph{Contribution.}
	We propose BuresGAN: a generative adversarial network which has the objective function of a vanilla GAN complemented by an additional term, given by the squared Bures distance between the covariance matrix of real and fake batches in a latent space. This loss function promotes a matching of fake and real data in a feature space $\mathbb{R}^f$, so that mode collapse is reduced. Conveniently, the Bures distance also admits both a feature space and kernel based expression. Contrary to other related approaches such as in~\cite{che2016mode} or \cite{srivastava2017veegan}, the architecture of the GAN is unchanged, only the objective is modified.
	We empirically show that the proposed method is robust when it comes to the choice of architecture and does not require an additional fine architecture search. Finally, an extra asset of BuresGAN is that it yields competitive or improved IS and FID scores compared with the state of the art on CIFAR-10 and STL-10 using a ResNet architecture.
	
	\paragraph{Related works.}
	The Bures distance is closely related to the Fr\'echet distance~\cite{DOWSON1982450} which is a $2$-Wasserstein distance between multivariate normal distributions. Namely, the Fr\'echet distance between multivariate normals of equal means is the Bures distance between their covariance matrices. The Bures distance is also equivalent to the exact expression for the $2$-Wasserstein distance between two elliptically contoured distributions with the same mean as shown in~\cite{Gelbrich} and  \cite{peyre2019computational}. Noticeably, it is also related to the Fr\'echet Inception Distance score (FID), which is a popular manner to assess the quality of generative models. This score uses the Fr\'echet distance between real and generated samples in the feature space of a pre-trained inception network as it is explained in~\cite{salimans2016improved} and \cite{NIPS2017_FID}.
	
	There exist numerous works aiming to improve training efficiency of generative networks. For mode collapse evaluation, we compare BuresGAN to the most closely related works. 
	GDPP-GAN  \cite{elfeki2018learning} and VEEGAN \cite{srivastava2017veegan} also try to enforce diversity in `latent' space.  GDPP-GAN matches the eigenvectors and eigenvalues of the real and fake diversity kernel. In VEEGAN, an additional reconstructor network is introduced to map the true data distribution to Gaussian random noise. In a similar way, architectures with two discriminators are analysed by~\cite{D2GAN}, while MADGAN~\cite{MADGAN} uses multiple discriminators and generators.
	A different approach is taken by Unrolled-GAN~\cite{metz2016unrolled} which updates the generator with respect to the unrolled optimization of the discriminator. This allows training to be adjusted between using the optimal discriminator in the generator’s objective, which is ideal but infeasible in practice. 
	Wasserstein GANs \cite{arjovsky2017wasserstein,gulrajani2017improved}  leverage the $1$-Wasserstein distance to match the real and generated data distributions. 
	In  MDGAN~\cite{che2016mode}, a regularization is added to the objective function, so that the generator can take advantage of another similarity metric with a  more predictable behavior. This idea is combined with a penalization of the missing modes. 
	Some other recent approaches to reducing mode collapse are variations of WGAN~\cite{Wu2018WassersteinDF}. Entropic regularization has been also proposed in PresGAN~\cite{PresGAN}, while metric embeddings were used in the paper introducing BourGAN~\cite{BourGAN}. A simple packing procedure which significantly reduces mode collapse was proposed in PacGAN~\cite{PacGAN} that we also consider hereafter in our comparisons.
	\section{Method}
	A GAN consists of a discriminator $D: \mathbb{R}^d \to \mathbb{R}$ and a generator $G: \mathbb{R}^\ell \to \mathbb{R}^d$ which are typically defined by neural networks and parametrized by real vectors. The value $D (\bm{x})$ gives the probability that $\bm{x}$ comes from the empirical distribution, while the generator $G$ maps a point $\bm{z}$ in the latent space $\mathbb{R}^\ell$ to a point in the input space $\mathbb{R}^d$.
	The training of a GAN consists in solving
	\begin{equation}
	\min_{G}\max_{D} \mathbb{E}_{\bm{x}\sim p_{d}}[\log D(\bm{x})] + \mathbb{E}_{\bm{x}\sim p_{g}}[\log (1-D(\bm{x}))],\label{eq:GAN}
	\end{equation}
	by alternating two phases of training. In \eqref{eq:GAN}, the expectation in the first term is over the empirical data distribution  $p_{d}$, while the expectation in the second term is over the generated data distribution $p_g$, implicitly given by the mapping by $G$ of the latent prior distribution $\mathcal{N}(0,\mathbb{I}_\ell)$. It is common to define and minimize the discriminator loss by
	\begin{equation}
	V_D = -\mathbb{E}_{\bm{x}\sim p_{d}}[\log D(\bm{x})] - \mathbb{E}_{\bm{x}\sim p_{g}}[\log (1-D(\bm{x}))].
	\end{equation}
	In practice, it is proposed in~\cite{Goodfellow} to minimize generator loss 
	\begin{equation}
	V_G =    -\mathbb{E}_{\bm{z}\sim \mathcal{N}(0,\mathbb{I}_\ell)}[\log D(G(\bm{z}))],
	\end{equation}
	rather than the second term of \eqref{eq:GAN}, for an improved training efficiency.
	\paragraph{Matching real and fake data covariance.}
	To prevent mode collapse, we encourage the generator to sample fake data of similar diversity to real data. This is achieved by matching the sample covariance matrices of real and fake data respectively. Covariance matching and similar ideas were explored for GANs in~\cite{McGAN} and \cite{elfeki2018learning}.
	In order to compare covariance matrices, we propose to use the squared Bures distance between positive semi-definite $\ell\times\ell$ matrices~\cite{BHATIA2019}, i.e.,
	\begin{align*}\mathcal{B}\left({A}, {B}\right)^{2} &=\min_{U\in O(\ell)}\|A^{1/2}-B^{1/2}U\|^2_F \\
	&=  \mathrm{Tr}({A}+{B}-2({A}^{\frac{1}{2}} {B}{A}^{\frac{1}{2}})^{\frac{1}{2}}).
	\end{align*}
	Being a Riemannian metric on the manifold of positive semi-definite matrices \cite{Massart}, the Bures metric is adequate to compare covariance matrices.  The covariances are defined in a feature space associated to the discriminator. Namely, let $D (\bm{x}) = \sigma(\bm{w}^\top \bm{\phi}(\bm{x}))$, where $\bm{w}$ is the weight vector of the last dense layer and $\sigma$ is the sigmoid function. The last layer of the discriminator, denoted by $\bm{\phi}(\bm{x})\in \mathbb{R}^{d_\phi}$, naturally defines  a feature map. We use the normalization $\bar{\bm{\phi}}(\bm{x}) = \bm{\phi}(\bm{x})/\|\bm{\phi}(\bm{x})\|_2$, after the centering of $\bm{\phi}(\bm{x})$. Then, we define a covariance matrix as follows:  $C(p) =  \mathbb{E}_{\bm{x}\sim p}[\bar{\bm{\phi}}(\bm{x})\bar{\bm{\phi}}(\bm{x})^\top]$.
	For simplicity, we denote the real data and generated data covariance matrices by $C_d =C(p_d)$ and $C_g =C(p_g)$, respectively.
	Our proposal is to replace the generator loss by 
	$V_G + \lambda \mathcal{B}(C_{d},C_{g})^2.
	$
	The value $\lambda=1$ was found to yield good results in the studied data sets.
	Two specific training algorithms are proposed. Algorithm \ref{alg:BuresGAN} deals with the squared Bures distance as an additive term to the generator loss, while an alternating training is discussed in Supplementary Material (SM) and does not introduce an extra parameter.
	\begin{algorithm}[H]
		\SetAlgoLined
		Sample a real and fake batch \;
		Update $G$ by minimizing $V_G +\lambda \mathcal{B}(\hat{\bm{C}}_{r}, \hat{\bm{C}}_{g})^{2}$\;
		Update $D$ by maximizing $- V_D$;\
		\caption{BuresGAN \label{alg:BuresGAN}}
	\end{algorithm}
	The training described in Algorithm \ref{alg:BuresGAN} is analogous to the training of GDPP GAN, although the additional generator loss is rather different. The computational advantage of the Bures distance is that it admits two expressions which can be evaluated numerically in a stable way. Namely, there is no need to calculate a gradient update through an eigendecomposition. 
	\paragraph{Feature space expression.} 
	In the training procedure, real $\bm{x}^{(d)}_i$ and fake data $\bm{x}^{(g)}_i$ with $1\leq i\leq b$ are sampled respectively from the empirical distribution and the mapping of the normal distribution $\mathcal{N}(0,\mathbb{I}_{\ell})$ by the generator. Consider the case where the batch size $b$
	is larger than the feature space dimension. Let the embedding of the batches in feature space be $\Phi_{\alpha} = [\bm{\phi}(\bm{x}^{(\alpha)}_1),\ldots,\bm{\phi}(\bm{x}^{(\alpha)}_b)]^\top \in \mathbb{R}^{b \times d_{\phi}}$ with $\alpha = d,g$.
	The covariance matrix of one batch in feature space\footnote{For simplicity, we omit the normalization by $\frac{1}{b-1}$ in front of the covariance matrix.} is
	$\hat{C} = \bar{\Phi}^{\top} \bar{\Phi}$, where  $\bar{\Phi}$ is the $\ell_2$-normalized centered feature map of the batch.
	Numerical instabilities can be avoided by adding a small number, e.g. $1\mathrm{e}{-14}$, to the diagonal elements of the covariance matrices, so that, in practice, we only deal with strictly positive definite matrices. From the computational perspective, an interesting alternative expression for the Bures distance is given by 
	\begin{equation}
	\mathcal{B}\left({C}_{d}, {C}_{g}\right)^{2} = \mathrm{Tr}\big({C}_{d}+{C}_{g}-2( {C}_{g} {C}_{d})^{\frac{1}{2}}\big),\label{eq:primal}
	\end{equation}
	whose computation requires only one matrix square root. This identity can be obtained from Lemma~\ref{lemma:sqrt}. Note that an analogous result is proved in~\cite{KernelWasserstein}.
	\begin{mylemma}
		Let $A$ and $B$ be 
		symmetric positive semidefinite matrices of the same size and let $B = Y^\top Y$. Then, we have: (i) $AB$ is diagonalizable with nonnegative eigenvalues, and (ii) $\Tr((AB)^{\frac{1}{2}}) = \Tr((Y A Y^\top)^{\frac{1}{2}})$. \label{lemma:sqrt}
	\end{mylemma}
	
	\begin{proof}(i) is a consequence of Corollary 2.3 in~\cite{HONG1991373}. (ii) We now follow~\cite{KernelWasserstein}. Thanks to (i), we have $AB = P D P^{-1}$ where $D$ is a nonnegative diagonal and the columns of $P$ contain the right eigenvectors of $AB$. Therefore, $\Tr((AB)^{1/2}) = \Tr(D^{1/2})$. Then, $Y A Y^\top$ is clearly diagonalizable. Let us show that it shares its nonzero eigenvalues with $AB$. a) We have $AB P = PD$, so that, by multiplying of the left by $Y$, it holds that  $(YAY^\top)Y P = YPD$. b) Similarly, suppose that we have the eigenvalue decomposition  $Y A Y^\top Q = Q\Lambda$. Then, we have $BAY^\top Q = Y^\top Q\Lambda$ with $B = Y^\top Y$. This means that the non-zero eigenvalues of $YAY^\top$ are also eigenvalues of $BA$. Since $A$ and $B$ are symmetric, this completes the proof.
	\end{proof}
	
	\paragraph{Kernel based expression.}
	Alternatively, if the feature space dimension $f$ is larger than the batch size $b$, it is more efficient to compute $\mathcal{B}(\hat{{C}}_{d}, \hat{{C}}_{g})$ thanks to $b\times b$ kernel matrices: $K_{d} = \bar{\Phi}_{d} \bar{\Phi}_d^{^\top}$, $K_{g} = \bar{\Phi}_{g} \bar{\Phi}_g^{\top}$ and
	$K_{dg} = \bar{\Phi}_d \bar{\Phi}^{\top}_g$.
	Then, we have the kernel based expression
	\begin{equation}
	\mathcal{B}(\hat{{C}}_{d}, \hat{{C}}_{g})^{2} =\Tr \big({K}_{d}+{K}_{g}-2\left( {K}_{dg} {K}_{dg}^\top \right)^{\frac{1}{2}}\big),\label{eq:dual}
	\end{equation}
	which allows to calculate the Bures distance between covariance matrices by computing a matrix square root of a $b\times b$ matrix. This is a consequence of Lemma~\ref{lemma:K}.
	\begin{mylemma}\label{lemma:K}
		The matrices $X^\top X Y^\top Y$ and $Y X^\top X Y^\top$ are diagonalizable with nonnegative eigenvalues and share the same non-zero eigenvalues.
	\end{mylemma}
	\begin{proof}
		The result follows from Lemma~1 and its proof, where $A=X^\top X$ and $B = Y^\top Y$. 
	\end{proof}
	\paragraph{Connection with Wasserstein GAN and integral probability metrics.}
	The Bures distance is proportional to the 2-Wasserstein distance $\mathcal{W}_2$ between two ellipically contoured distributions with the same mean \cite{Gelbrich}. For instance, in the case of multivariate normal distributions, we have
	\[
	\mathcal{B}\left({A}, {B}\right)^{2} = \min_{\pi} \mathbb{E}_{(X,Y)\sim \pi}\|X-Y\|_2^2 \text{ s.t.}
	\begin{cases}
	X\sim \mathcal{N}(0,A)\\
	Y\sim \mathcal{N}(0,B),
	\end{cases}
	\]
	where the minimization is over the joint distributions $\pi$. More precisely, in this paper, we make the approximation that the implicit distribution of the real and generated data in the feature space $\mathbb{R}^{d_\phi}$ (associated to $\bm{\phi}(\bm{x})$) are elliptically contoured with the same mean. Under different assumptions, the Generative Moment Matching Networks \cite{MomentMatching,MMDGAN} work in the same spirit, but use a different approach to match covariance matrices.
	On the contrary, WGAN uses the Kantorovich dual formula for the 1-Wasserstein distance: 
	$\mathcal{W}_1(\alpha,\beta)=\sup_{f\in {\rm Lip }}\int f\diff (\alpha-\beta)$, where $\alpha,\beta$ are signed measures.  Generalizations of such integral formulae are called integral probability metrics (see for instance~\cite{binkowski2018demystifying}). Here, $f$ is the discriminator, so that the maximization over Lipschitz functions  $f$ plays the role of the maximization over discriminator parameters in the min-max game of~\eqref{eq:GAN}. Then, in the training procedure, this maximization alternates with a minimization over the generator parameters.
	
	We can now discuss the connection with Wasserstein GAN.
	Coming back to the definition of BuresGAN, we can now explain that the 2-Wasserstein distance provides an upper bound on an integral probability metric. Then, if we assume that the densities are elliptically contoured distributions in feature space, the use of the Bures distance to calculate $\mathcal{W}_2$ allows to spare the maximization over the discriminator parameters -- and this motivates why the optimization of $\mathcal{B}$ only influences updates of the generator in Algorithm \ref{alg:BuresGAN}.
	Going more into detail, the 2-Wasserstein distance between two probability densities (w.r.t. the same measure) is equivalent to a Sobolev dual norm, which can be interpreted as an integral probability metric. Indeed, let the Sobolev semi-norm $\| f\|_{H^1} = (\int \|\nabla f(x)\|^2 \diff x )^{1/2} $. Then, its dual norm over signed measures is defined as 
	$
	\| \nu\|_{H^{-1}} = \sup_{\|f\|_{H^1} \leq 1}\int  f \diff\nu 
	$.
	It is then shown in~\cite{PeyreESAIM} and \cite{peyre2019computational} that there exist two positive constants $c_1$ and $c_2$ such that 
	\[
	c_1 \| \alpha-\beta\|_{H^{-1}}\leq \mathcal{W}_2(\alpha,\beta)\leq c_2 \| \alpha-\beta\|_{H^{-1}}.
	\]
	Hence, the $2$-Wasserstein distance gives an upper bound on an integral probability metric.
	
	\paragraph{Algorithmic details.}
	The matrix square root in \eqref{eq:primal} and \eqref{eq:dual} is obtained thanks to the Newton-Schultz algorithm which is inversion free and can be efficiently calculated on GPUs since it involves only matrix products.
	In practice, we found $15$ iterations of this algorithm to be sufficient for the small scale data sets, while $20$ iterations were used for the ResNet examples.
	A small regularization term $1\mathrm{e}{-14}$ is added to the matrix diagonal for stability. The latent prior distribution is $\mathcal{N}(0,\mathbb{I}_\ell)$ with $\ell = 100$ and the parameter in Algorithm \ref{alg:BuresGAN} is always set to $\lambda =1$.
	In the tables hereafter, we indicate the largest scores in bold, although we invite the reader to also consider the standard deviation.
	\section{Empirical Evaluation of Mode Collapse \label{sec:ModeCollapse}}
	BuresGAN's performances on synthetic data, artificial and real images are compared with the standard DCGAN ~\cite{salimans2016improved}, WGAN-GP, MDGAN, Unrolled GAN, VEEGAN,  GDPP and PacGAN.  We want to emphasize that the purpose of this experiment is not to challenge these baselines, but to report the improvement obtained by adding the Bures metric to the objective function. It would be straightforward to add the Bures loss to other GAN variants, as well as most GAN architectures, and we would expect an improvement in mode coverage and generation quality. In the experiments, we notice that adding the Bures loss to the vanilla GAN already significantly improves the results.
	
	A low dimensional feature space ($d_\phi=128$) is used for the synthetic data so that the feature space formula in~\eqref{eq:primal} is used, while the dual formula in~\eqref{eq:dual} is used for the image data sets (Stacked MNIST, CIFAR-10, CIFAR-100 and STL-10) for which the feature space is larger than the batch size. The architectures used for the image data sets are based on DCGAN~\cite{RadfordMC15}, while results using ResNets are given in Section~\ref{sec:ResNet}. All images are scaled in between -1 and 1 before running the algorithms. Additional information on the architectures and data sets is given in SM. The hyperparameters of other methods are typically chosen as suggested in the authors' reference implementation. The number of unrolling steps in Unrolled GAN is chosen to be $5$. For MDGAN, both versions are implemented but only MDGAN-v2 gives interesting results. The first version, which corresponds to the mode regularizer, has hyperparameters $\lambda_1 = 0.2$ and $\lambda_2 = 0.4$, for the second version, which corresponds to manifold diffusion training for regularized GANs, has $\lambda = 10^{-2}$. WGAN-GP uses $\lambda=10.0$ and $n_{\mathrm{critic}}= 5$. 
	All models are trained using Adam~\cite{kingma2014adam} with $\beta_1 = 0.5$, $\beta_2=0.999$ and learning rate $10^{-3}$ for both the generator and discriminator. Unless stated otherwise, the batch size is $256$. Examples of random generations of all the GANs are given in SM.  
	Notice that in this section we report the results achieved only at the end of the training.
	\subsection{Artificial Data}
	\paragraph{Synthetic.}
	\textsc{Ring} is a mixture of eight two-dimensional isotropic Gaussians in the plane with means $2.5\times(\cos((2\pi/8)i), \sin((2\pi/8)i))$ and std $0.01$ for $1\leq i\leq 8$. \textsc{Grid} is a mixture of $25$ two-dimensional isotropic normals in the plane with means separated by $2$ and with standard deviation $0.05$. All models have the same architecture, with $\ell = 256$ following~\cite{elfeki2018learning}, and are trained for $25$k iterations. The evaluation is done by sampling $3$k points from the generator network. A sample is counted as high quality if it is within $3$ standard deviations of the nearest mode. The experiments are repeated $10$ times for all models and their performance is compared in Table~\ref{tab:Artificial}.
	
	BuresGAN consistently captures all the modes and produces the highest quality samples. The training progress of the BuresGAN is shown on Figure~\ref{fig:progressSynt:AltBures}, where we observe that all the modes early on in the training procedure, afterwards improving the quality. The training progress of the other GAN models listed in Table~\ref{tab:Artificial} is given in SM. Although BuresGAN training times are larger than most other methods for this low dimensional example, we show in SM that BuresGAN scales better with the input data dimension and architecture complexity.
	%
	\begin{figure*}[h]
		\centering
		\begin{tabular}{@{}c@{}}
			\includegraphics[width=1\linewidth]{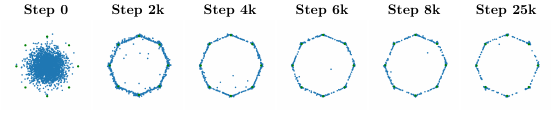}
		\end{tabular}
		\begin{tabular}{@{}c@{}}
			\includegraphics[width=1\linewidth]{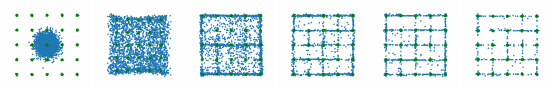}
		\end{tabular}
		\caption{Figure accompanying Table \ref{tab:Artificial}, the progress of BuresGAN on the synthetic examples. Each column shows $3$k samples from the training of the generator in blue and $3$k samples from the true distribution in green.}
		\label{fig:progressSynt:AltBures}
	\end{figure*}

	\begin{table}[h]
		\caption{Experiments on the synthetic data sets. Average (std) over 10 runs. All the models are trained for 25k iterations.}
		
		\label{tab:Artificial}
		\begin{center}

			\begin{tabular}{l l l l l }
				\toprule
				& \multicolumn{2}{c}{Grid (25 modes)} & \multicolumn{2}{c}{Ring  (8 modes)}\\
				\cmidrule(lr){2-3}\cmidrule(lr){4-5}
				& Nb modes & $\%$ in $3\sigma$ & Nb modes & $\%$ in $3\sigma$ \\
				\cmidrule(lr){2-2}  \cmidrule(lr){3-3}\cmidrule(lr){4-4} \cmidrule(lr){5-5}
				GAN  & $22.9(4)$ & $76(13)$   & $7.4(2)$ & $76(25)$ \\
				WGAN-GP  &$24.9(0.3)$ & $77(10)$   & $7.1(1)$ & $9(5)$ \\
				MDGAN-v2  &$\bm{25}(0)$ & $68(11)$    & $5(3)$ & $20(15)$\\
				Unrolled  &$19.7(1)$ & $78(19)$  & $\bm{8}(0)$ & $77(18)$  \\
				VEEGAN &$\bm{25}(0)$ & $67(3)$   & $\bm{8}(0)$ & $29(5)$ \\
				GDPP  & $20.5(5)$& $ 79(23)$    & $7.5(0.8)$ & $73(25)$  \\
				PacGAN2  & $23.6 (4)$ & $65 (28)$   & $\bm{8} (0)$  & $81 (15)$  \\
				BuresGAN (ours) & $\bm{25}(0)$ & $\bm{82}(1)$  & $ \bm{8}(0)$ & $\bm{82}(4)$\\
				\bottomrule
			\end{tabular}

		\end{center}
	\end{table}
	\paragraph{Stacked MNIST.}
	The Stacked MNIST data set is specifically constructed to contain $1000$ known modes. This is done by stacking three digits, sampled uniformly at random from the original MNIST data set, each in a different channel. 
	BuresGAN is compared to the other methods and are trained for $25$k iterations. 
	For the performance evaluation, we follow~\cite{metz2016unrolled} and use the following metrics:  the number of captured modes measures mode collapse and the KL divergence, which also measures sample quality. The mode of each generated image is identified by using a standard MNIST classifier which is trained up to $98.43\%$ accuracy on the validation set
	(see Supplementary Material), and classifies each channel of the fake sample. 
	The same classifier is used to count the number of captured modes. The metrics are calculated based on $10$k generated images for all the models.
	Generated samples from BuresGAN are given in Figure \ref{fig:Mixed}.
	\begin{figure*}[h]
		\centering
		\includegraphics[width=\linewidth]{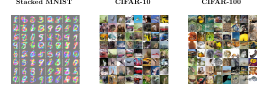}
		\caption{Generated samples from a trained BuresGAN, with a DCGAN architecture.
		}
		\label{fig:Mixed}
	\end{figure*}
	As it was observed by previous papers such as in~\cite{PacGAN} and~\cite{BourGAN}, even the vanilla GAN can achieve an excellent mode coverage for certain architectures.
	
	To study mode collapse on this data set, we performed extensive experiments for multiple discriminator layers and multiple batch sizes. In the main section, the results for a 3 layer discriminator are reported in Table~\ref{tab:SMNIST_3}. Additional experiments can be found in SM. An analogous experiment as in  VEEGAN~\cite{srivastava2017veegan} with a challenging architecture including 4 convolutional layers for both the generator and discriminator was also included for completeness; see Table~\ref{tab:SMNIST_VEEGAN}. Since different authors, such as in the PacGAN's paper, use slightly different setups, we also report in SM the specifications of the different settings. We want to emphasize the interests of this simulation is to compare GANs in the same consistent setting, while the results may vary from those reported in e.g. in~\cite{srivastava2017veegan} since some details might differ.

	\begin{table}[h]
		\centering
		\caption{KL-divergence between the generated  distribution and true distribution for an architecture with 3 conv. layers for the Stacked MNIST dataset. The number of counted modes assesses mode collapse. The results are obtained after $25$k iterations and we report the average(std) over $10$ runs.
			\label{tab:SMNIST_3}}
		\begin{center}

			\resizebox{\textwidth}{!}{
				
				\begin{tabular}{r l c c c c c c}
					\toprule
					& & \multicolumn{3}{c}{Nb modes($\uparrow$)} & \multicolumn{3}{c}{KL div.($\downarrow$)} \\ 
					\cmidrule(lr){3-5}\cmidrule(lr){6-8}
					&  Batch size & 64 & 128 & 256  & 64 & 128 & 256 \\ 
					\cmidrule(lr){2-2}  \cmidrule(lr){3-3}\cmidrule(lr){4-4} \cmidrule(lr){5-5}\cmidrule(lr){6-6}\cmidrule(lr){7-7}\cmidrule(lr){8-8} 
					\multirow{10}{*}{\rotatebox{90}{3 conv. layers}} &  GAN & $993.3(3.1)$& $995.4(1.7)$  & $\bm{998.3}(1.2)$  &$ 0.28(0.02)$ & $\bm{0.24}(0.02)$ &$ 0.21(0.02) $\\
					& WGAN-GP & $980.2(57)$ & $838.3(219)$  & $785.1(389)$  & $\bm{0.26}(0.34)$ & $1.05(1)$   & $1.6(2.4)$ \\
					&MDGAN-v1 & $233.8(250)$ &  $204.0(202)$ & $215.5(213)$  & $5.0(1.6)$ & $4.9(1.3)$  &  $5.0(1.2)$ \\
					&MDGAN-v2 & $299.9(457)$ & $300.4(457)$  & $200.0(398)$  & $4.8(3.0)$ & $4.7(3.0)$  & $5.5(2.6)$  \\
					& UnrolledGAN  & $934.7(107)$ & $874.1(290)$  & $884.9(290)$  & $0.72(0.51)$ & $0.98(1.46)$  & $0.90(1.4)$ \\
					&VEEGAN  & $974.2(10.3)$ & $687.9(447)$   & $395.6(466)$  & $0.33(0.05)$ & $2.04(2.61)$  & $3.52(2.64)$  \\
					& GDPP  & $894.2(298)$ & $897.1(299)$   &$997.5(1.4) $ & $0.92(1.92)$ & $0.88(1.93) $ &$\bm{0.20}(0.02)$  \\
					& PacGAN2  & $  989.8(4.0)$  & $993.3 (4.8)$   & $897.7(299)$  & $0.33(0.02)$  & $ 0.29(0.04)$  &  $0.87(1.94)$ \\
					& BuresGAN (ours)   &  $\bm{993.5}(2.7)$ & $\bm{996.3}(1.6)$  & $997.1(2.4)$  & $0.29(0.02)$ & $0.25(0.02)$  & $0.23(0.01)$  \\
					\bottomrule
				\end{tabular}
			}

		\end{center}
	\end{table}

	\begin{table}[h]
		\centering
		\caption{ Stacked MNIST experiment for an architecture with 4 conv. layers. All the models are trained for $25$k iterations with a batch size of $64$, a learning rate of $2\times 10^{-4}$ for Adam and a normal latent distribution. The evaluation is over $10$k samples and we report the average(std) over $10$ runs.
			\label{tab:SMNIST_VEEGAN}}
		\begin{center}

			\begin{tabular}{r c c c}
				\toprule
				& & Nb modes($\uparrow$) & KL div.($\downarrow$)  \\
				\cmidrule(lr){2-4}
				\multirow{9}{*}{\rotatebox{90}{4 conv. layers}} &  GAN & $ 21.6(25.8) $ & $ 5.10(0.83) $\\%
				& WGAN-GP  & $ \bm{999.7}(0.6)  $  & $ \bm{0.11}(0.006) $ \\
				& MDGAN-v2  & $ 729.5(297.9)  $  & $ 1.76(1.65) $ \\
				& UnrolledGAN  & $ 24.3(23.61)$  & $ 4.96(0.68) $ \\
				& VEEGAN  & $ 816.1(269.6)  $  & $ 1.33(1.46) $ \\
				& GDPP  & $ 33.3(39.4)  $  & $ 4.92(0.80) $ \\
				& PacGAN2  & $ 972.4(12.0) $  & $ 0.45(0.06) $ \\
				& BuresGAN (ours)   &  $ 989.9(4.7) $ & $ 0.38(0.06)  $ \\
				\bottomrule
			\end{tabular}

		\end{center}
	\end{table}
	
	Interestingly, for most models, an improvement is observed in the quality of the images -- KL divergence -- and in terms of mode collapse -- number of modes attained -- as the size of the batch increases. For the same batch size, architecture and iterations, the image quality is improved by BuresGAN, which consistently performs well over all settings. The other methods show a higher variability over the different experiments. MDGANv2, VEEGAN, GDPP and WGAN-GP often have an excellent single run performance. However, when increasing the number of discriminator layers, the training of these models has a tendency to collapse more often as indicated by the large standard deviation.
	Vanilla GAN is one of the best performing models in the variant with $3$ layers. 
	
	We observe in Table~\ref{tab:SMNIST_VEEGAN} for the additional experiment that vanilla GAN and GDPP collapse for this architecture. WGAN-GP yields the best result and is followed by BuresGAN. However, as indicated in Table~\ref{tab:SMNIST_3}, WGAN-GP is sensitive to the choice of architecture and hyperparameters and its training time is also longer as it can be seen from the corresponding timings table in SM. More generally, these results depend heavily on the precise architecture choice and to a lesser extent on the batch size. These experiments further confirm the finding that most GAN models, including the standard version, can learn all modes with careful and sufficient architecture tuning~\cite{PacGAN,numericsofgans}.
	Finally, it can be concluded that BuresGAN 
	performs well for all settings, showing that it is robust when it comes to batch size and architecture.

	\subsection{Real Images}
	\paragraph{Metrics.}
	Image quality is assessed thanks to the Inception Score (IS), Fr\'echet Inception Distance (FID) and Sliced Wasserstein Distance (SWD). The latter was also used in~\cite{elfeki2018learning} and~\cite{karras2017progressive} to evaluate image quality as well as mode-collapse. In a word, SWD evaluates the multiscale statistical similarity between distributions of local image patches drawn from Laplacian pyramids. A small Wasserstein distance indicates that the distribution of the
	patches is similar, thus real and fake images appear similar in both appearance and variation at this spatial resolution. 
	The metrics are calculated based on $10$k generated images for all the models.
	\paragraph{CIFAR data sets.}
	We evaluate the GANs on the $32\times 32\times 3$ CIFAR data sets, for which all models are trained for 100k iterations with a convolutional architecture. In Table~\ref{tab:CIFAR},
	the best performance is observed for BuresGAN in terms of image quality, measured by FID and Inception Score, and in terms of mode collapse,  measured by SWD. We also notice that UnrolledGAN, VEEGAN and WGAN-GP have difficulty converging to a satisfactory result for this architecture. This is contrast to the `simpler' synthetic data and the Stacked MNIST data set, where the models attain a performance comparable to  BuresGAN.  Also, for this architecture and number of training iterations, MDGAN did not converge to a meaningful result in our simulations. In~\cite{arjovsky2017wasserstein}, WGAN-GP achieves a very good performance on CIFAR-10 with a ResNet architecture which is considerably more complicated than the DCGAN used here. Therefore, results with a Resnet architecture are reported in Section~\ref{sec:ResNet}.
	
	\begin{table}[h]
		\caption{Generation quality on CIFAR-10, CIFAR-100 and STL-10 with DCGAN architecture. For the CIFAR images, Average(std) over 10 runs and $100$k iterations for each. For the STL-10 images, Average(std) over 5 runs and $150$k iterations for each.
			For improving readability, SWD score was multiplied by $100$. The symbol `$\sim$' indicates that no meaningful result could be obtained for these parameters.}
		\label{tab:CIFAR}
		\begin{center}

			\resizebox{\textwidth}{!}{
				
				\begin{tabular}{l c c c c  c c  c cc}
					\toprule
					& \multicolumn{3}{c}{CIFAR-10} & \multicolumn{3}{c}{CIFAR-100} & \multicolumn{3}{c}{STL-10}\\ 
					\cmidrule(lr){2-4}\cmidrule(lr){5-7} \cmidrule(lr){8-10}
					& IS$(\uparrow$) & FID$(\downarrow$) &  SWD$(\downarrow$)  & IS$(\uparrow$) & FID$(\downarrow$) &  SWD$(\downarrow$) & IS$(\uparrow$) & FID$(\downarrow$) &  SWD$(\downarrow$) \\
					\cmidrule(lr){2-2}  \cmidrule(lr){3-3}\cmidrule(lr){4-4} \cmidrule(lr){5-5}\cmidrule(lr){6-6}\cmidrule(lr){7-7} \cmidrule(lr){8-8}\cmidrule(lr){9-9}\cmidrule(lr){10-10}
					GAN   & $5.67(0.22)$ & $59(8.5)$ & $3.7(0.9)$  & $5.2(1.1)$ & $91.7(66)$ &$7.8(4.9)$ & $2.9(1.8)$& $237(54)$  & $12.3(4.1)$ \\
					WGAN-GP   &  $2.01(0.47)$  & $291(87)$ & $8.3(1.9)$  & $1.2(0.5)$ & $283(113)$ & $9.7(2.5)$ & $\sim$ & $\sim$ & $\sim$\\
					UnrolledGAN  & $3.1(0.6)$  & $148(42)$  & $9.0(5)$  & $3.2(0.7)$  & $172.9(40)$  & $13.1(9.2)$ & $\sim$ & $\sim$ & $\sim$\\
					VEEGAN & $2.5(0.6)$ & $198(33.5)$ &  $12.0(3)$  &$2.8(0.7)$  & $177.2(27)$ & $12.8(3.9)$ & $\sim$ & $\sim$ & $\sim$\\
					GDPP   & $5.76(0.27)$ & $62.1(5.5)$  & $4.1(1.1)$  & $5.9(0.2)$ & $65.0(8)$ & $4.4(1.9)$ & $3.3(2.2)$& $232(84)$ & $8.2(4.0)$\\ 
					PacGAN2     & $5.51(0.18)$ & $60.61(5.9)$ & $2.95(1)$  & $5.6 (0.1)$ & $59.9 (5.2)$ & $4.0 (1.8)$ & $4.7(1.5)$ & $161(36)$& $8.1(4.3)$\\
					BuresGAN (ours)   & $\bm{6.34}(0.17)$  & $\bm{43.7}(0.9)$ & $\bm{2.1}(0.6)$   & $\bm{6.5}(0.1)$  & $\bm{47.2}(1.2)$ & $\bm{2.1}(1.0)$ & $\bm{7.6}(0.3)$& $\bm{109}(7)$ & $\bm{2.3}(0.3)$\\
					\bottomrule
				\end{tabular}
			}

		\end{center}
	\end{table}
	CIFAR-100 data set consists of $100$ different classes and is therefore more diverse. 
	Compared to the original CIFAR-10 data set, the performance of the studied GANs remains almost the same. An exception is vanilla GAN, which shows a higher presence of mode collapse as measured by SWD.

	\paragraph{STL-10.}
	
	\begin{figure*}[h]
		\centering
		\includegraphics[width=0.4\linewidth]{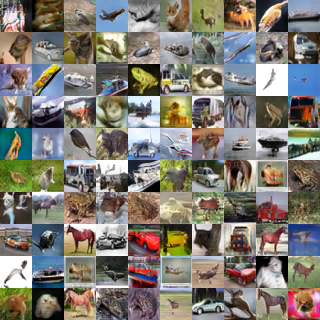}
		\qquad
		\includegraphics[width=0.4\linewidth]{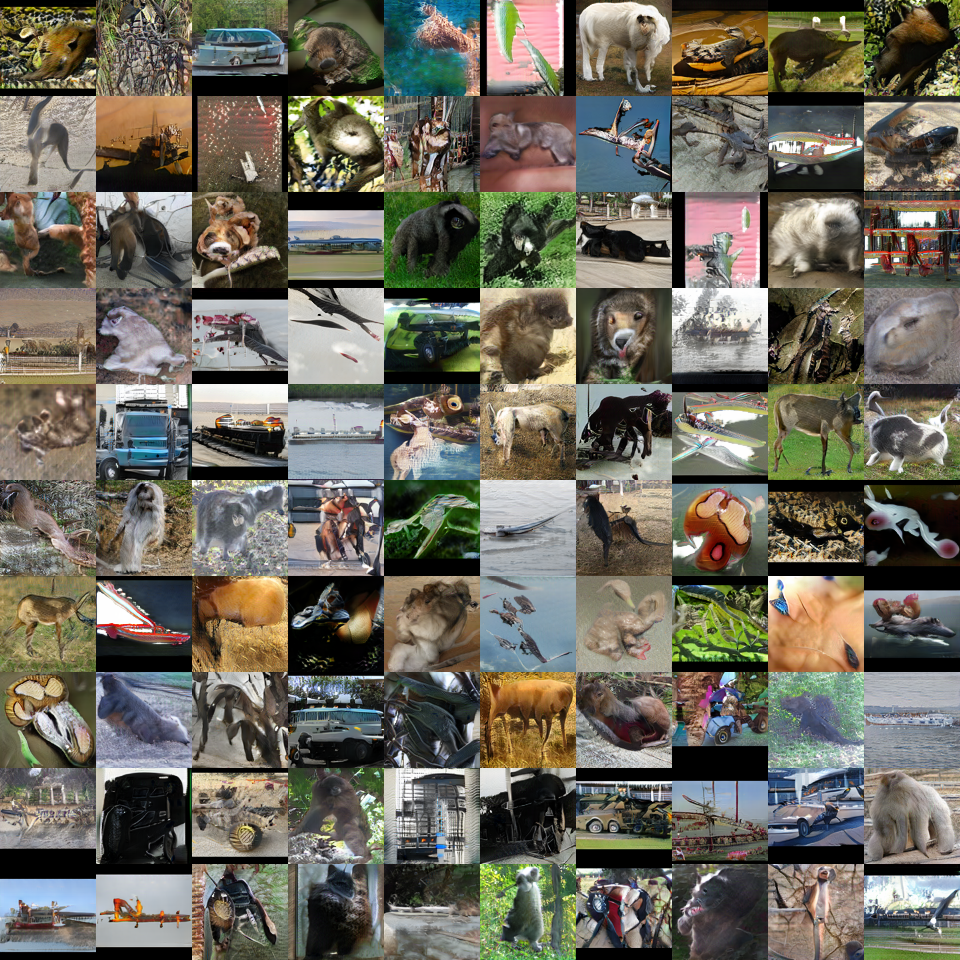}
		\caption{Images generated by BuresGAN with a ResNet architecture for {CIFAR-10} (left) and {STL-10} (right) data sets. The STL-10 samples are full-sized $96\times 96$ images.}
		\label{fig:STL10}
	\end{figure*}

	The STL-10  data set 
	includes higher resolution images of size $96 \times 96 \times3$.
	The best performing models from previous experiments are trained for $150$k iterations. Samples of generated  images from BuresGAN are given on Figure \ref{fig:STL10}. 
	The metrics are calculated based on $5$k generated images for all the models. Compared to the previous data sets, GDPP and vanilla GAN are rarely able to generate high quality images on the higher resolution STL-10  data set. Only BuresGANs are capable of consistently generating high quality images as well as preventing mode collapse, for the same architecture.
	%

	%
	%
	%
	%
	%
	\paragraph{Timings.}
	The computing times for these data sets are in SM. For the same number of iterations, BuresGAN training time is comparable to WGAN-GP training for the simple data in Table~\ref{tab:Artificial}. For more complicated architectures, BuresGAN scales better and the training time was observed to be significantly shorter with respect to WGAN-GP and several other methods.
	
	\section{High Quality Generation using a ResNet Architecture \label{sec:ResNet}}
	
	\begin{table}[h]
		\caption{Best achieved IS and FID, using a ResNet architecture. Results with an asterisk are quoted from their respective papers (std in parenthesis). BuresGAN results were obtained after $300k$ iterations and averaged over 3 runs. The result indicated with $\dagger$ are taken from~\cite{Wu2018WassersteinDF}. For all the methods, the STL-10 images are rescaled to $48\times 48\times 3$ in contrast with Table~\ref{tab:CIFAR}.}
		\label{tab:resnet}
		\begin{center}

			\begin{tabular}{l c c c c}
				\toprule
				&  \multicolumn{2}{c}{CIFAR-10}   & \multicolumn{2}{c}{STL-10} \\
				\cmidrule(lr){2-3} \cmidrule(lr){4-5} 
				&  IS ($\uparrow$) & FID ($\downarrow$) &  IS ($\uparrow$) & FID ($\downarrow$) \\
				\cmidrule(lr){2-2} \cmidrule(lr){3-3} \cmidrule(lr){4-4} \cmidrule(lr){5-5} 
				WGAN-GP ResNet~\cite{gulrajani2017improved}$^*$ & $7.86(0.07)$ & $18.8^{\dagger}$ & / & / \\
				InfoMax-GAN~\cite{InfoMax-GAN}$^*$ & $8.08(0.08)$ & $17.14 (0.20)$ & $8.54(0.12)$ & $37.49 (0.05)$\\
				SN-GAN ResNet~\cite{miyato2018spectral}$^*$  & $8.22(0.05)$ & $21.7(0.21)$ & $9.10(0.04)$ & $40.1 (0.50)$ \\
				ProgressiveGAN~\cite{karras2017progressive}$^*$  & $8.80(0.05)$ & / & / & /\\
				CR-GAN~\cite{zhang2019consistency}$^*$  & $8.4$ & $14.56$ & / & /\\
				NCSN~\cite{song2019generative}$^*$ &  $\bm{8.87}(0.12)$ & $25.32$ & / & /  \\
				Improving MMD GAN~\cite{wang2018improving}$^*$ &  $8.29$ & $16.21$ & $9.34$ & $37.63$  \\
				WGAN-div~\cite{Wu2018WassersteinDF}$^*$ & / & $18.1^{\dagger}$ & /  & / \\
				BuresGAN ResNet (Ours) &  $8.81(0.08)$ & $\bm{12.91} (0.40)$ & $\bm{9.67}(0.19)$ & $\bm{31.42}(1.01)$ \\
				\bottomrule
			\end{tabular}

		\end{center}
	\end{table}

	As noted by~\cite{lucic2018gans}, a fair comparison should involve GANs with the same architecture, and this is why we restricted in our paper to a classical DCGAN architecture. It is natural to question the performance of BuresGAN with 
	a ResNet architecture. Hence, we trained BuresGAN on the CIFAR-10 and STL-10 data sets by using the ResNet architecture taken from~\cite{gulrajani2017improved}. In this section, the STL-10 images are rescaled to a resolution of $48\times 48 \times 3$ according the procedure described in~\cite{miyato2018spectral,InfoMax-GAN,wang2018improving}, so that the comparison of IS and FID scores with other works is meaningful.   Note that BuresGAN has no parameters to tune, except for the hyperparameters of the optimizers. 
	
	The results are displayed in Table~\ref{tab:resnet}, where the scores of state-of-the-art unconditional GAN models with a ResNet architecture are also reported. 
	In contrast with Section~\ref{sec:ModeCollapse}, we report here the best performance achieved at any time during the training, averaged over several runs. To the best of our knowledge, our method achieves a new state of the art inception score on STL-10 and is within a standard deviation of state of the art on CIFAR-10 using a ResNet architecture. The FID score achieved by BuresGAN is nonetheless smaller than the reported FID scores for GANs using a ResNet architecture. A visual inspection of the generated images in Figure~\ref{fig:STL10} shows that the high inception score is warranted, the samples are clear, diverse and often recognizable. BuresGAN also performs well on the full-sized STL-10 data set where an inception score of $11.11 \pm{0.19}$ and an FID of $50.9 \pm{0.13}$ is achieved (average and std over $3$ runs).
	\section{Conclusion}
	In this work, we discussed an additional term based on the Bures distance which 
	promotes a matching of the distribution of the generated and real data in a feature space $\mathbb{R}^{d_\phi}$. The Bures distance admits both a feature space and kernel based expression, which makes the proposed model time and data efficient when compared to state of the art models.
	Our experiments show that the proposed methods are capable of reducing mode collapse and, on the real data sets, achieve a clear improvement of sample quality without parameter tuning and without the need for regularization such as a gradient penalty. Moreover, the proposed GAN shows a stable performance over different architectures, data sets and hyperparameters.  
	
	\FloatBarrier
	\bibliographystyle{unsrt}
	\bibliography{Archive}

	\appendix
	\section{Organization}
	In Section~\ref{sec:TheoDetails},  some details about our theoretical results are given. Next, additional experiments are described in Section~\ref{sec:AddExp}. Importantly, the latter section includes
	\begin{itemize}
		\item another training algorithm (Alt-BuresGAN), which is parameter-free and achieves comptetitive results with BuresGAN,
		\item an ablation study comparing the effect of the Bures metric with the Frobenius norm,
		\item timings per iteration for all the generatives methods listed in this paper,
		\item supplementary tables for the experiments described in the paper,
		\item a more extensive study on Stacked MNIST of the influence of the batch size for a simpler DCGAN architecture with 2 convolutional layers in the discriminator.
	\end{itemize}
	Next, illustrative figures including generated images are listed in Section~\ref{sec:AddFig}. Finally, Section~\ref{sec:Arch} gives details about the architectures used in this paper.
	
	
	\section{Details of the theoretical results\label{sec:TheoDetails}}
	Let $A$ and $B$ be symmetric and positive semi-definite matrices. Let $A^{1/2}= U\diag(\sqrt{\lambda})U^\top$ where $U$ and $\lambda$ are obtained thanks to the eigenvalue decomposition $A = U \diag(\lambda) U^\top$.  We show here that the Bures distance between $A$ and $B$ is
	\begin{equation}
	\mathcal{B}\left({A}, {B}\right)^{2} =\min_{U\in O(\ell)}\|A^{1/2}-B^{1/2}U\|^2_F =  \mathrm{Tr}({A}+{B}-2({A}^{\frac{1}{2}} {B}{A}^{\frac{1}{2}})^{\frac{1}{2}}),\label{eq:SuppBures0}
	\end{equation}
	
	where $O(\ell)$ is the set of $\ell\times \ell$ orthonormal matrices.
	We can simplify the above expression as follows
	\begin{equation}
	\min_{U\in O(\ell)}\|A^{1/2}-B^{1/2}U\|^2_F = \Tr(A)+\Tr(B) -2\max_{U\in O(\ell)}\Tr(A^{1/2}B^{1/2}U)\label{eq:SuppBures}
	\end{equation}
	since $\Tr(U^\top B^{1/2}A^{1/2}) = \Tr(A^{1/2}B^{1/2}U)$. Let the characteristic function of the set of orthonormal matrices be  $f(U) = \chi_{O(\ell)}(U)$ that is, $f(U)=0$ if $U\in O(\ell)$ and $+\infty$ otherwise.
	\setcounter{mylemma}{1}
	\begin{mylemma}
		The Fenchel conjugate of $f(U) = \chi_{O(\ell)}(U)$ is $f^\star(M) = \|M\|_\star$,
		where the nuclear norm is $\|M\|_\star = \Tr(\sqrt{M^\top M})$ and $U,M\in\mathbb{R}^{\ell\times\ell}$.
	\end{mylemma}
	\begin{proof}
		The definition of the Fenchel conjugate with respect to the Frobenius inner product gives
		\[
		f^\star(M) =  \sup_{U\in\mathbb{R}^{\ell\times\ell}} \Tr(U^\top M) - f(U) = \max_{U\in O(\ell)}\Tr(U^\top M).
		\]
		Next we decompose $M$ as follows: $M = W \Sigma V^\top$, where $W,V\in O(\ell)$ are orthogonal matrices and $\Sigma$ is a $\ell\times\ell$ diagonal matrix with non negative diagonal entries, such that $MM^\top = W \Sigma^2 W^\top$ and $M^\top M = V \Sigma^2 V^\top$. Notice that the non zero diagonal entries of $\Sigma$ are the singular values of $M$.
		Then, 
		\[
		\max_{U\in O(\ell)}\Tr(U^\top M) = \max_{U\in O(\ell)}\Tr( W \Sigma V^\top U) = \max_{U'\in O(\ell)}\Tr(\Sigma U'),
		\]
		where we renamed $U' = V^\top U W$. Next, we remark that  $\Tr(\Sigma U')=\Tr(\Sigma \diag(U'))$. Since by construction,  $\Sigma$ is diagonal with non negative entries the maximum is attained at $U' = \mathbb{I}$. Then, the optimal objective is $\Tr(\Sigma) = \Tr(\sqrt{M^\top M})$.
	\end{proof}
	By taking $M = A^{1/2}B^{1/2}$ we obtain \eqref{eq:SuppBures0}. Notice that the role of $A$ and $B$ can be exchanged in \eqref{eq:SuppBures0} since $U$ is orthogonal.

	\FloatBarrier
	\section{Additional Experiments \label{sec:AddExp}}
	\subsection{Alternating training}
	In the sequel, we also provide the results obtained thanks to Alt-BuresGAN, described in Algorithm~\ref{alg:Alt-BuresGAN}, which is parameter-free and yield competitive results with BuresGAN. Nevertheless, the training time of Alt-BuresGAN is larger compared to BuresGAN.
	\begin{algorithm}[H]
		\SetAlgoLined
		Sample a real and fake batch \;
		Update $G$ by minimizing $\mathcal{B}(\hat{\bm{C}}_{r}, \hat{\bm{C}}_{g})^{2}$\;
		Update $G$ by minimizing $V_G$\;
		Update $D$ by maximizing $ -V_D$;\
		\caption{Alt-BuresGAN \label{alg:Alt-BuresGAN}}
	\end{algorithm}
	
	\subsection{Another matrix norm}
	In this section, we also compare BuresGAN with a similar construction using another matrix norm. Precisely, we also use the squared Frobenius norm (Frobenius$^2$GAN), i.e., with the following generator loss 
	$V_G + \lambda \|C_{d}-C_{g}\|_{F}^2$, 
	with $\lambda = 1$.
	We observe that
	\begin{itemize}
		\item Frobenius$^2$GAN also performs well on simple image data sets such as Stacked MNIST,
		\item Frobenius$^2$GAN has a worse performance with respect to BuresGAN on CIFAR and STL-10 data sets (see Table~\ref{tab:CIFAR_othernorms}). 
	\end{itemize}
	This observation emphasizes that the use of the Bures metric is requisite in order to obtain high quality samples for higher dimensional image data sets such as CIFAR and STL-10 and that the simpler variant based on the Frobenius norm cannot achieve similar result.
	
	\subsection{Timings \label{app:Exp_timings}}
	The timings per iteration for the experiments presented in the paper are listed in Table \ref{tab:timings}.
	Times are given for all the methods considered, although some methods do not always generate meaningful images for all datasets. The methods are timed for $50$ iterations after the first $5$ iterations, which is used as a warmup period, following which the average number of iterations per second is computed.
	The fastest method is the vanilla GAN. BuresGAN has a similar computation cost as GDPP. We observe that (alt-)BuresGAN is significantly faster compared to WGAN-GP.
	In order to obtain reliable timings, these results were obtained on the same GPU Nvidia Quadro P4000, although, for convenience, the experiments on these image datasets were executed on a machines equipped with different GPUs.
	\begin{table}[h]
		\caption{Average time per iteration in seconds for the convolutional architecture. Average over 5 runs, with std in parenthesis. The batch size is $256$.  For Stacked MNIST, we use a discriminator architecture with 3 convolutional layers (see Table~\ref{tab:gen_dis_architecture_smnist}). For readability, the computation times for the Grid and Ring dataset are multiplied by $10^{-3}$.
		}
		\label{tab:timings}
		\begin{center}

			\resizebox{\textwidth}{!}{
				\begin{tabular}{l c c c c c c }
					\toprule
					& {Grid} & {Ring} & {stacked MNIST}& {CIFAR-10} & {CIFAR-100} & {STL-10} \\ 
					\cmidrule(lr){2-2} \cmidrule(lr){3-3} \cmidrule(lr){4-4} \cmidrule(lr){5-5} \cmidrule(lr){6-6} \cmidrule(lr){7-7}
					GAN & $\bm{1.28}  (0.008)$ &    $\bm{1.28}  (0.016)$ & $\bm{0.54}(0.0005)$ & $ \bm{0.65} (0.02)$ & $\bm{0.64} (0.0008)$ &$\bm{6.00} (0.01)$\\
					WGAN-GP &  $9.0   (0.04)$ &   $8.96  (0.04)$  & $2.99(0.004)$ & $3.41 (0.009) $ & $3.41 (0.006)$ & $36.5 (0.03)$ \\
					UnrolledGAN &  $3.92  (0.008)$ &  $3.92  (0.028)$ & $1.90(0.002)$& $2.17 (0.003) $ & $2.18 (0.004)$ & $21.99 (0.06)$ \\
					MDGAN-v2 &  $3.24  (0.008)$ &  $3.24  (0.032)$   & $1.66 (0.002)$ & $1.98 (0.002) $ & $1.98 (0.002)$ & $18 (0.03)$ \\
					VEEGAN &  $1.96  (0.012)$ &  $1.96  (0.02)$  & $0.56 (0.006)$ & $0.66$ (0.006)  & $0.65 (0.004)$ & $6.10 (0.03)$\\
					GDPP & $15.68  (1.0)$  &  $16.16  (0.2)$ & $0.69(0.02)$ & $0.80 (0.02)$ & $0.80 (0.02) $ & $7.46 (0.03)$\\
					PacGAN2 &  $1.72  (0.016)$ &  $1.72  (0.024)$  & $0.77(0.006)$ &  $0.91(0.007) $ & $ 0.91(0.007)$ & $ 8.02(0.008)$\\
					\cmidrule(lr){1-7}
					Frobenius$^2$GAN &  $2.32  (0.024)$ &  $2.28  (0.028)$ & $ 0.98 (0.004) $ & $ 1.09 (0.004) $ & $ 1.09(0.003) $ & $ 13.6(0.01) $ \\
					BuresGAN & $11.16  (0.016)$ & $11.92  (0.04)$ & $0.72 (0.02)$ & $0.82 (0.001)$ & $0.82 (0.0008)$ & $7.6 (0.03)$ \\
					Alt-BuresGAN & $11.92  (0.024)$ & $11.12  (0.04)$ & $0.98 (0.008)$& $1.15 (0.007) $ & $1.15 (0.007)$ & $10.10 (0.03)$\\
					\bottomrule
				\end{tabular}
			}

		\end{center}
	\end{table}
	\subsection{Best inception scores achieved with DCGAN architecture}
	The inception scores for the best trained models are listed in Table \ref{tab:BestRuns}.
	For the CIFAR datasets, the largest inception score is significantly better than the mean for UnrolledGAN  and VEEGAN. This is the same for GAN and GDPP on the STL-10 dataset, where the methods often converge to bad results.  Only the proposed methods are capable of consistently generating high quality images over all datasets.
	
	\begin{table}[h]
		\caption{Inception Score for the best trained models on CIFAR-10, CIFAR-100 and STL-10, with a DCGAN architecture (higher is better). }
		\label{tab:BestRuns}
		\begin{center}

			\begin{tabular}{l  c c c }
				\toprule
				& {CIFAR-10} & {CIFAR-100} & {STL-10} \\ 
				\cmidrule(lr){2-2} \cmidrule(lr){3-3} \cmidrule(lr){4-4} 
				GAN   & 5.92 & 6.33 & 6.13 \\
				WGAN-GP    & 2.54 & 2.56 & / \\
				UnrolledGAN   & 4.06 & 4.14  &  /\\
				VEEGAN   & 3.51  & 3.85 & / \\
				GDPP   & 6.21 & 6.32 & 6.05 \\
				\cmidrule(lr){1-4}
				BuresGAN   & 6.69 & 6.67 &  7.94 \\
				Alt-BuresGAN   & 6.40  & 6.48 &  7.88 \\
				\bottomrule
			\end{tabular}

		\end{center}
	\end{table}
	
	\subsection{Influence of the number of convolutional layers for DCGAN architecture}
	Results for an architecture with 3 and 4 convolutional layers described in Table~\ref{tab:gen_dis_architecture_smnist} and Table~\ref{tab:SMNIST_VEEGAN} are listed in the paper. An overview of similar experiments in other papers is given in Table~\ref{tab:SMNISTparameters} hereafter.
	A more complete study is performed in this section.
	We provide here results with a DCGAN architecture for a simpler architecture with 2 convolutional layers for the discriminator (see, Table~\ref{tab:SMNIST_2} for the results and Table~\ref{tab:gen_dis_architecture_smnist} for the architecture).

	\paragraph{Discussion.} Interestingly, for most models, an improvement is observed in the quality of the images -- KL divergence -- and in terms of mode collapse -- number of modes attained -- as the size of the batch increases. For the same batch size, architecture and iterations, the image quality is improved by BuresGAN, which is robust with respect to batch size and architecture choice. The other methods show a higher variability over the different experiments. 
	WGAN-GP has the best single run performance with a discriminator with $3$ convolutional layers and has on average a superior performance when using a discriminator with $2$ convolutional layers (see Table~\ref{tab:SMNIST_2}) but sometimes fails to converge when increasing the number of discriminator layers by $1$ along with increasing the batch size. 
	MDGANv2, VEEGAN, GDPP and WGAN-GP often have an excellent single run performance. However, when increasing the number of discriminator layers, the training of these models has a tendency to collapse more often as indicated by the large standard deviation.
	Vanilla GAN is one of the best performing models in the variant with $3$ layers. This indicates that, for certain datasets, careful architecture tuning can be more important than complicated training schemes. A lesson from Table~\ref{tab:SMNIST_2} is that BuresGAN's mode coverage does not vary much if the batch size increases, although the KL divergence seems to be slightly improved.
	\paragraph{Ablation study.} Results obtained with the Frobenius norm rather than the Bures metric are also included in the tables of this section. 
	\begin{table}[h]
		\caption{Comparison for Frobenius norm. Experiments on the synthetic data sets. Average (std) over 10 runs. All the models are trained for 25k iterations.}
		\label{tab:Artificial_App}
		\begin{center}

			\begin{tabular}{l l l l l }
				\toprule
				& \multicolumn{2}{c}{Grid (25 modes)} & \multicolumn{2}{c}{Ring  (8 modes)}\\
				\cmidrule(lr){2-3}\cmidrule(lr){4-5}
				& Nb modes & $\%$ in $3\sigma$ & Nb modes & $\%$ in $3\sigma$ \\
				\cmidrule(lr){2-2}  \cmidrule(lr){3-3}\cmidrule(lr){4-4} \cmidrule(lr){5-5}
				Frobenius$^2$GAN & $ 25 (0) $ & $ 73 (1) $  & $ 8 (0)  $ & $ 59 (16) $\\         
				BuresGAN  & $25(0)$ & $82(1)$  & $ 8(0)$ & ${82}(4)$\\
				Alt-BuresGAN & $\bm{25}(0)$ & $\bm{84}(1)$  & $ \bm{8}(0)$ & $\bm{84}(6)$ \\
				\bottomrule
			\end{tabular}

		\end{center}
	\end{table}
	\begin{table}[h]
		\caption{ KL-divergence between the generated  distribution and true distribution (Quality, lower is better). The number of counted modes indicates the amount of mode collapse (higher is better). $25$k iterations and average and std over $10$ runs. Same architecture as in Table~2 in the paper with a discriminator with $2$ convolutional layers. The architecture details are in Table~\ref{tab:gen_dis_architecture_smnist}. \label{tab:SMNIST_2}}
		\begin{center}

			\begin{tabular}{r l c c c c c c}
				\toprule
				& & \multicolumn{3}{c}{Nb modes($\uparrow$)} & \multicolumn{3}{c}{KL div.($\downarrow$)} \\ 
				\cmidrule(lr){3-5}\cmidrule(lr){6-8}
				&  Batch size & 64 & 128 & 256  & 64 & 128 & 256 \\ 
				\cmidrule(lr){2-2}  \cmidrule(lr){3-3}\cmidrule(lr){4-4} \cmidrule(lr){5-5}\cmidrule(lr){6-6}\cmidrule(lr){7-7}\cmidrule(lr){8-8} 
				\multirow{10}{*}{\rotatebox{90}{2 conv. layers}} &    GAN & $970.5(5.8)$& $972.7(6.4)$ & $979(3.5)$  & $0.47(0.04)$& $0.44(0.02)$ & $0.41(0.03)$  \\
				& WGAN-GP & $\bm{996.7}(1.6)$ & $\bm{997.5}(0.9)$  & $\bm{998.1}(1.5)$   & $\bm{0.25}(0.01)$ &$\bm{0.22}(0.01)$ & $\bm{0.21}(0.05)$  \\
				& MDGAN-v1  & $115.9(197)$ & $260.9(267)$ & $134.3(157)$ & $5.5(1.4)$ & $4.9(1.7)$ & $5.8(0.9)$  \\
				& MDGAN-v2  & $698.1(456)$ & $898.4(299)$ & $599.2(488)$ & $2.2(3.0)$ & $0.86(1.9)$ & $2.8(3.2)$   \\
				& UnrolledGAN  & $953.5(11)$ & $971.4(4.8)$ & $966.2(17.3)$ & $0.71(0.06)$ & $0.60(0.04)$ & $0.58(0.10)$  \\
				& VEEGAN  & $876.7(290)$ & $688.5(443)$ & $775.9(386)$ & $0.92(1.6)$ & $1.9(2.4)$ & $1.54(2.2)$  \\
				&  GDPP  & $974.4(3.3)$ & $978.2(7.6)$ & $980.5(6.0)$ & $0.45(0.02)$ & $0.43(0.03)$ & $0.41(0.03)$  \\
				&  PacGAN2  & $969.8(6.9)$ & $971.1(3.6) $ & $977.9(4.3) $ & $0.54(0.04) $ & $0.51(0.01) $ & $0.48(0.02) $  \\
				\cmidrule(lr){2-8}
				& Frobenius$^2$GAN   &  $ 975.8(8.0) $ & $ 981.0(4.1)$  & $  981.8(3.2)  $  & $ 0.34(0.03) $ & $0.28(0.02) $  & $ 0.25(0.01) $  \\ 
				&  BuresGAN   & $973.2(1.3)$ & $979.9(4.0)$ &  $981.1(4.9)$ & $0.36(0.02)$ & $0.30(0.02)$ & $0.25(0.01)$  \\
				&   Alt-BuresGAN & $975.4(6.8)$& $978.2(5.4)$  & $980.2(3.0)$   & $0.37(0.02)$& $0.30(0.01)$ & $0.28(0.01)$   \\
				\bottomrule
			\end{tabular}

		\end{center}
	\end{table}

	\begin{table}[h]
		\caption{Ablation study on real image data sets for a convolutional architecture. Frobenius norm is also considered in the place of the Bures metric. Generation quality on CIFAR-10, CIFAR-100 and STL-10 with a DCGAN architecture. Average(std) over 10 runs. $100$k iterations for each.
			For improving readability, SWD score was multiplied by $100$.}
		\label{tab:CIFAR_othernorms}
		\begin{center}

			\resizebox{\textwidth}{!}{
				\begin{tabular}{l c c c c  c c c c  c c c c}
					\toprule
					& \multicolumn{3}{c}{CIFAR-10} & \multicolumn{3}{c}{CIFAR-100} & \multicolumn{3}{c}{STL-10}\\ 
					\cmidrule(lr){2-4}\cmidrule(lr){5-7} \cmidrule(lr){8-10}
					& IS$(\uparrow$) & FID$(\downarrow$) &  SWD$(\downarrow$)  & IS$(\uparrow$) & FID$(\downarrow$) &  SWD$(\downarrow$) & IS$(\uparrow$) & FID$(\downarrow$) &  SWD$(\downarrow$)\\ 
					\cmidrule(lr){2-2}  \cmidrule(lr){3-3}\cmidrule(lr){4-4} \cmidrule(lr){5-5}\cmidrule(lr){6-6}\cmidrule(lr){7-7} \cmidrule(lr){8-8}\cmidrule(lr){9-9}\cmidrule(lr){10-10}
					Frobenius$^2$GAN & $5.40(0.16)$ & $57.7(4.8)$ & $\bm{1.3}(0.3)$ & $5.52(0.16)$ & $60.1(4.4)$ & $\bm{1.6}(0.5)$ & $6.26 (0.15)$ & $132.18 (3.6)$ & $\bm{1.8} (0.4)$\\
					BuresGAN    & $\bm{6.34}(0.17)$  & $\bm{43.7}(0.9)$ & $2.1(0.6)$   & $\bm{6.5}(0.1)$  & $\bm{47.2}(1.2)$ & $2.1(1.0)$ & $\bm{7.6}(0.3)$& $\bm{109}(7)$ & $2.3(0.3)$\\
					Alt-BuresGAN  & $6.23(0.07)$ & $45.4(2.8)$  & $1.7(0.9)$  & $6.4(0.1)$ & $49.4(3.4)$ & $1.8(0.6)$ & $7.5(0.3)$& $110(4)$ & $2.8(0.4)$\\
					\bottomrule
				\end{tabular}
			}

		\end{center}
	\end{table}
	\begin{table}[h]
		\centering
		\caption{Ablation study for Stacked MNIST experiment for an architecture with 4 convolutional layers (see Tbale~\ref{tab:gen_dis_architecture_smnist_4layers}). All the models are trained for $25$k iterations with a batch size of $64$, a learning rate of $2\mathrm{e}{-4}$ for Adam and a normal latent distribution. The evaluation is over $10$k samples and we report an average (std) over $10$ runs. 
			\label{tab:SMNIST_VEEGAN_App}}
		\begin{center}

			\begin{tabular}{r c c c}
				\toprule
				& & Nb modes($\uparrow$) & KL div.($\downarrow$)  \\
				\cmidrule(lr){3-4}
				& Frobenius$^2$GAN  &  $984.7(6.4)  $ & $ 0.38(0.04)$ \\
				& BuresGAN   &  $ 989.9(4.7) $ & $ 0.38(0.06)  $ \\
				&  Alt-BuresGAN & $990.0(4.9) $ & $ 0.33(0.04) $ \\
				\bottomrule
			\end{tabular}

		\end{center}
	\end{table}
	\FloatBarrier
	\section{Additional Figures\label{sec:AddFig}}
	\subsection{Training visualization at different steps for the artificial data sets}
	In Figure~\ref{fig:progressSynt:ring} and Figure~\ref{fig:progressSynt:grid}, we provide snapshots of generated samples for the methods studied in the case of \textsc{Ring} and \textsc{Grid}. In both cases, BuresGAN seems to generate samples in such a way that the symmetry of the training data is preserved during training. This is also true for WGAN and MDGAN, while some other GANs seem to recover the symmetry of the problem at the end of the training phase.
	\subsection{Generated images of the trained GANs}
	Next, in Figure~\ref{fig:stacked_ALL}, Figure~\ref{fig:CIFAR10_ALL}, Figure~\ref{fig:CIFAR100_ALL} and Figure~\ref{fig:STL10_ALL}, we list samples of generated images for the models trained with a DCGAN architecture for Stacked MNIST, CIFAR-10, CIFAR-100 and STL-10 respectively.
	\begin{figure}[h]
		\centering
		\includegraphics[width=0.9\textwidth,height=0.9\textheight,keepaspectratio]{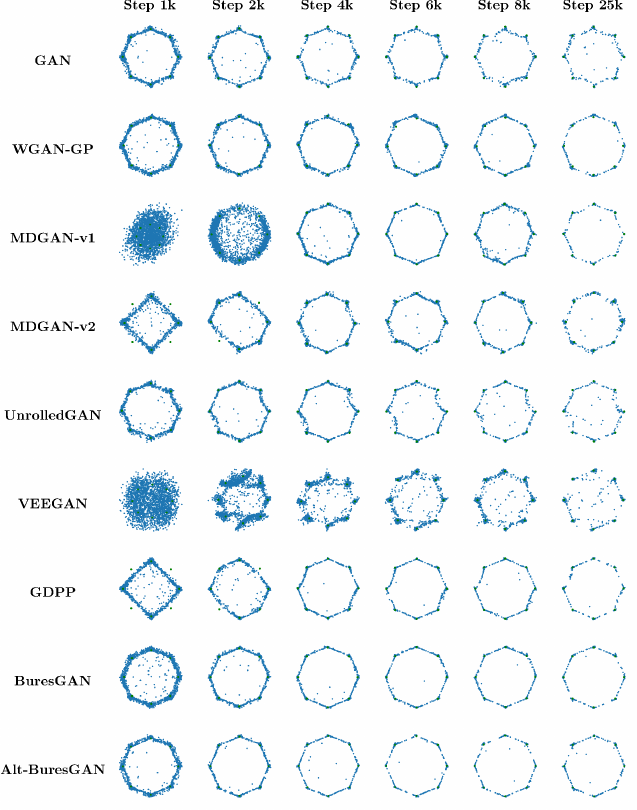}
		\caption{The progress of different GANs on the synthetic ring example. Each column show $3000$ samples from the training generator in blue with $3000$ samples from the true distribution in green.}
		\label{fig:progressSynt:ring}
	\end{figure}
	
	\begin{figure}[h]
		\centering
		\includegraphics[width=0.9\textwidth,height=0.9\textheight,keepaspectratio]{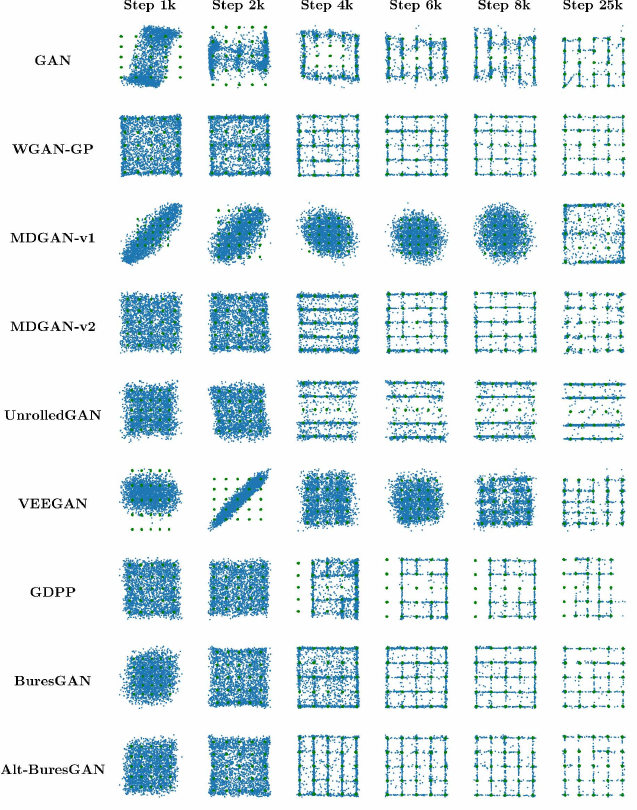}
		\caption{The progress of different GANs on the synthetic grid example. Each column show $3000$ samples from the training generator in blue with $3000$ samples from the true distribution in green.}
		\label{fig:progressSynt:grid}
	\end{figure}
	
	\begin{figure}[h]
		\centering
		\includegraphics[width=0.9\textwidth,height=0.9\textheight,keepaspectratio]{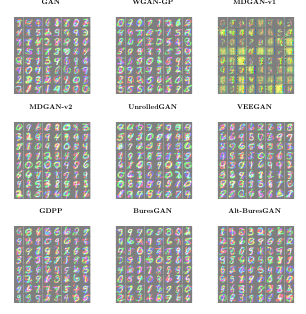}
		\caption{Generated images for the Stacked MNIST dataset. Each model is trained with $3$ layers and mini-batch size $256$. Each square shows $64$ samples from the trained generator.}
		\label{fig:stacked_ALL}
	\end{figure}
	
	\begin{figure}[h]
		\centering
		\includegraphics[width=0.9\textwidth,height=0.9\textheight,keepaspectratio]{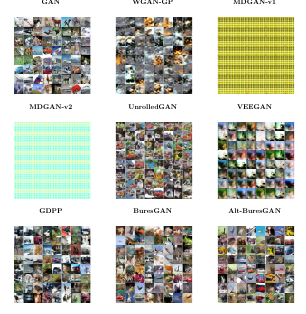}
		\caption{Generated images for CIFAR-10 using a DCGAN architecture. Each square shows $64$ samples from the trained  generator.}
		\label{fig:CIFAR10_ALL}
	\end{figure}
	
	\begin{figure}[h]
		\centering
		\includegraphics[width=0.9\textwidth,height=0.9\textheight,keepaspectratio]{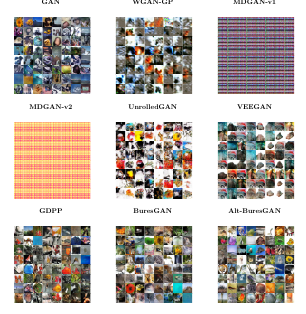}
		\caption{Generated images for CIFAR-100 using a DCGAN architecture. Each square shows $64$ samples from the trained  generator.}
		\label{fig:CIFAR100_ALL}
	\end{figure}

	\begin{figure}[h]
		\centering
		\includegraphics[width=0.9\textwidth,height=0.9\textheight,keepaspectratio]{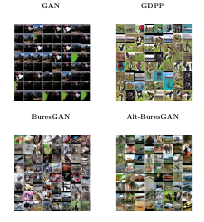}
		\caption{Generated images for STL-10 using a DCGAN architecture. Each square shows $64$ samples from the trained generator.}
		\label{fig:STL10_ALL}
	\end{figure}
	
	\FloatBarrier
	
	\section{Architectures \label{sec:Arch}}
	
	\subsection{Synthetic Architectures}
	Following the recommendation in the original work of~\cite{srivastava2017veegan}, the same fully-connected architecture is used for the VEEGAN reconstructor in all experiments.
	
	\begin{table}[h]
		\caption{The generator and discriminator architectures for the synthetic examples.}
		\label{tab:gen_dis_architecture_synthetic}
		\begin{center}

			\begin{tabular}{l c c c }
				\toprule
				Layer & Output & Activation  \\
				\cmidrule(lr){1-3}
				Input  & 256  & -  \\
				Dense  & 128  & tanh  \\
				Dense  & 128  &  tanh \\
				Dense   &  2  &  - \\
				\bottomrule
			\end{tabular} \quad \quad
			\begin{tabular}{l c c c }
				\toprule
				Layer & Output & Activation  \\
				\cmidrule(lr){1-3}
				Input  & 2  & -  \\
				Dense  & 128  & tanh  \\
				Dense  & 128  &  tanh \\
				Dense   &  1  &  - \\
				\bottomrule
			\end{tabular}

		\end{center}
	\end{table}
	
	\begin{table}[h]
		\caption{Respectively the MDGAN encoder model and VEEGAN stochastic inverse generator architectures for the synthetic examples. The output of the VEEGAN models are samples drawn from a normal distribution with scale 1 and where the location is learned.}
		\label{tab:encoder_architecture_synthetic}
		\begin{center}

			\begin{tabular}{l c c c }
				\toprule
				Layer & Output & Activation  \\
				\cmidrule(lr){1-3}
				Input  & 2  & -  \\
				Dense  & 128  & tanh  \\
				Dense   &  256  &  - \\
				\bottomrule
			\end{tabular} \quad \quad
			\begin{tabular}{l c c c }
				\toprule
				Layer & Output & Activation  \\
				\cmidrule(lr){1-3}
				Input  & 2  & -  \\
				Dense  & 128  & tanh  \\
				Dense  & 128  &  tanh \\
				Dense  & 256  &  tanh \\
				Normal   &  256  &  - \\
				\bottomrule
			\end{tabular}

		\end{center}
	\end{table}

	\subsection{Stacked MNIST Architectures}
	For clarity, we report in Table~\ref{tab:SMNISTparameters} the settings of different papers for the Stacked MNIST experiment. This helps to clarify the difference between our simulations and other papers' simulations.
	On the one hand, in Table~\ref{tab:gen_dis_architecture_smnist_4layers}, the architecture of the experiment reported in the paper is described. On the other hand, the architecture of the supplementary experiments on Stacked MNIST is given in Table~\ref{tab:gen_dis_architecture_smnist}. Details of the architecture of MDGAN are also reported in Table~\ref{tab:encoder_architecture_smnist}.
	\begin{table}[h]
		\caption{The hyperparameters and architectures of the stacked MNIST experiment performed by the different methods are given here. This represents our best efforts at providing a comparison between the different parameters used. The asterisks $^{**}$ indicate that the parameters where obtained by the Github repository. Notice that GDPP paper also used $30000$ iterations for training DCGAN and unrolled GAN (indicated by $^*$). BuresGAN's column refers to the settings of experiment of Table~2 in the paper and Table~\ref{tab:SMNIST_VEEGAN} in SM respectively, for which the different values are separated by the symbol \&.}
		\label{tab:SMNISTparameters}
		\begin{center}

			\begin{tabular}{l c c c c}
				\toprule
				Parameters  & BuresGAN & PacGAN & VEEGAN & GDPP\\
				\midrule
				learning rates & $1 \times 10^{-3}$ \& $2 \times 10^{-4}$  & $2 \times 10^{-4}$ & $2 \times 10^{-4}$ & $1 \times 10^{-4}$ $^{**}$ \\
				learning rate decay & no & no & no  &   no$^{**}$ \\
				Adam $\beta_1$  & 0.5 & 0.5 & 0.5 & 0.5\\
				Adam $\beta_2$  & 0.999 & 0.999 & 0.999 & 0.9\\
				iterations & 25000 & 20000 & ? & 15000$^*$\\
				disc. conv. layers & 3 \& 4 & 4 & 4 & 3\\
				gen. conv. layers & 3 \& 4 & 4 & 4 & 3\\
				batch size  & 64, 128, 256 \& 64 & 64 & 64 & 64\\
				evaluation samples  & 10000 & 26000 & 26000 & 26000\\
				$\ell$ (i.e., z dimension) & 100 & 100 & 100 & 128\\
				z distribution  & $\mathcal{N}(0,\mathbb{I}_\ell)$ & unif$[-1,1]^\ell$ & unif$[-1,1]^\ell$ $^{**}$ & unif$[-1,1]^\ell$\\
				\bottomrule
			\end{tabular}

		\end{center}
	\end{table}
	\begin{table}[h]
		\caption{The generator and discriminator architectures for the Stacked MNIST experiments with 4 layers for both generator and discrimintor. The BN column indicates whether batch normalization is used after the layer or not.}
		\label{tab:gen_dis_architecture_smnist_4layers}
		\begin{center}

			\begin{tabular}{l c c c c}
				\toprule
				Layer & Output & Activation & BN \\
				\cmidrule(lr){1-4}
				Input  & 100  & -  & -\\
				Dense  & 2048  & ReLU  & Yes\\
				Reshape   &  2, 2, 512  &  - & -\\
				Conv'   &  4, 4, 256  &  ReLU & Yes\\
				Conv'   &  7, 7, 128  &  ReLU & Yes\\
				Conv'   &  14, 14, 64  &  ReLU & Yes\\
				Conv'   &  28, 28, 3  &  ReLU & Yes\\
				\bottomrule
			\end{tabular} \quad \quad
			\begin{tabular}{l c c c c}
				\toprule
				Layer & Output & Activation & BN \\
				\cmidrule(lr){1-4}
				Input  & 28, 28, 3  & -  & -\\
				Conv  & 14, 14, 64  & Leaky ReLU  & No \\
				Conv  & 7, 7, 128  &  Leaky ReLU & Yes\\
				Conv  & 4, 4, 256  &  Leaky ReLU & Yes\\
				Conv  & 2, 2, 512  &  Leaky ReLU & Yes\\
				Flatten  & - &  - & - \\
				Dense  & 1  &  - & -\\
				\bottomrule
			\end{tabular}

		\end{center}
	\end{table}
	\begin{table}[h]
		\caption{The generator and discriminator architectures for the Stacked MNIST experiments. The BN column indicates whether batch normalization is used after the layer or not. For the experiments with 2 convolution layers in Table~\ref{tab:SMNIST_2}, the final convolution layer is removed in the discriminator.}
		\label{tab:gen_dis_architecture_smnist}
		\begin{center}

			\begin{tabular}{l c c c c}
				\toprule
				Layer & Output & Activation & BN \\
				\cmidrule(lr){1-4}
				Input  & 100  & -  & -\\
				Dense  & 12544  & ReLU  & Yes\\
				Reshape   &  7, 7, 256  &  - & -\\
				Conv'   &  7, 7, 128  &  ReLU & Yes\\
				Conv'   &  14, 14, 64  &  ReLU & Yes\\
				Conv'   &  28, 28, 3  &  ReLU & Yes\\
				\bottomrule
			\end{tabular} \quad \quad
			\begin{tabular}{l c c c c}
				\toprule
				Layer & Output & Activation & BN \\
				\cmidrule(lr){1-4}
				Input  & 28, 28, 3  & -  & -\\
				Conv  & 14, 14, 64  & Leaky ReLU  & No \\
				Conv  & 7, 7, 128  &  Leaky ReLU & Yes\\
				Conv  & 4, 4, 256  &  Leaky ReLU & Yes\\
				Flatten  & - &  - & - \\
				Dense  & 1  &  - & -\\
				\bottomrule
			\end{tabular}

		\end{center}
	\end{table}
	\begin{table}[h]
		\caption{The CNN architecture of the classifier used during the evaluation of the stackedMNIST experiments. Dropout with a rate of 0.5 is used before the final dense layer.}
		\label{tab:MNISTclassifier}
		\begin{center}
			
			\begin{tabular}{l c c}
				\toprule
				Layer & Output & Activation \\
				\cmidrule(lr){1-3}
				Input  & 28, 28,1  & -  \\
				Conv   &  24, 24, 32  &  ReLU \\
				MaxPool   &  12, 12, 32  &  - \\
				Conv   &  8, 8, 64  &  ReLU \\
				MaxPool   &  4, 4, 64  &  - \\
				Flatten   &  -  &  - \\
				Dense   &  1024  &  ReLU \\
				Dense   &  10  &  - \\
				\bottomrule
			\end{tabular} \quad \quad
			
		\end{center}
	\end{table}
	\begin{table}[h]
		\caption{The MDGAN encoder model architecture for the Stacked MNIST experiments. The BN column indicates whether batch normalization is used after the layer or not.}
		\label{tab:encoder_architecture_smnist}
		\begin{center}

			\begin{tabular}{l c c c c}
				\toprule
				Layer & Output & Activation & BN \\
				\cmidrule(lr){1-4}
				Input  & 28, 28, 3  & -  & -\\
				Conv   &  14, 14, 3  &  ReLU & Yes\\
				Conv   &  7, 7, 64  &  ReLU & Yes\\
				Conv   &  7, 7, 128  &  ReLU & Yes\\
				Flatten   &  -  &  - & -\\
				Dense  & 100  &  - & -\\
				\bottomrule
			\end{tabular}

		\end{center}
	\end{table}
	\FloatBarrier
	\subsection{CIFAR-10 and 100 DCGAN Architectures}
	Convolutional architectures for CIFAR data set are reported in Table~\ref{tab:gen_dis_architecture_cifar} and Table~\ref{tab:encoder_architecture_cifar} below.
	\begin{table}[h]
		\caption{The generator and discriminator architectures for the CIFAR-10 and CIFAR-100 experiments. The BN column indicates whether batch normalization is used after the layer or not.}
		\label{tab:gen_dis_architecture_cifar}
		\begin{center}

			\begin{tabular}{l c c c c}
				\toprule
				Layer & Output & Activation & BN \\
				\cmidrule(lr){1-4}
				Input  & 100  & -  & -\\
				Dense  & 16384  & ReLU  & Yes\\
				Reshape   &   8, 8, 256  &  - & -\\
				Conv'   &  8, 8, 128  &  ReLU & Yes\\
				Conv'   &  16, 16, 64  &  ReLU & Yes\\
				Conv'   &  32, 32, 3  &  ReLU & Yes\\
				\bottomrule
			\end{tabular} \quad \quad
			\begin{tabular}{l c c c c}
				\toprule
				Layer & Output & Activation & BN \\
				\cmidrule(lr){1-4}
				Input  & 32, 32, 3  & -  & -\\
				Conv  & 16, 16, 64  & Leaky ReLU  & No \\
				Conv  & 8, 8, 128  &  Leaky ReLU & Yes\\
				Conv  & 4, 4, 256  &  Leaky ReLU & Yes\\
				Flatten  & - &  - & - \\
				Dense  & 1  &  - & -\\
				\bottomrule
			\end{tabular}

		\end{center}
	\end{table}
	\begin{table}[h]
		\caption{The MDGAN encoder model architecture for the CIFAR-10 and CIFAR-100 experiments. The BN column indicates whether batch normalization is used after the layer or not.}
		\label{tab:encoder_architecture_cifar}
		\begin{center}

			\begin{tabular}{l c c c c}
				\toprule
				Layer & Output & Activation & BN \\
				\cmidrule(lr){1-4}
				Input  & 32, 32, 3  & -  & -\\
				Conv   &  16, 16, 3  &  ReLU & Yes\\
				Conv   &  8, 8, 64  &  ReLU & Yes\\
				Conv   &  8, 8, 128  &  ReLU & Yes\\
				Flatten   &  -  &  - & -\\
				Dense  & 100  &  - & -\\
				\bottomrule
			\end{tabular}

		\end{center}
	\end{table}
	\FloatBarrier
	
	\subsection{STL-10 DCGAN Architectures}
	In Table~\ref{tab:gen_dis_architecture_stl}, the convolutional architecture  is described for the $96\times 96\times 3$ STL-10 data set, while MDGAN details are reported in Table~\ref{tab:encoder_architecture_stl}.
	\begin{table}[h]
		\caption{The generator and discriminator architectures for the STL-10  experiments. The BN column indicates whether batch normalization is used after the layer or not.}
		\label{tab:gen_dis_architecture_stl}
		\begin{center}

			\begin{tabular}{l c c c c}
				\toprule
				Layer & Output & Activation & BN \\
				\cmidrule(lr){1-4}
				Input  & 100  & -  & -\\
				Dense  & 36864  & ReLU  & Yes\\
				Reshape   &    12, 12, 256  &  - & -\\
				Conv'   &  12, 12, 256  &  ReLU & Yes\\
				Conv'   &  24, 24, 128  &  ReLU & Yes\\
				Conv'   &  48, 48, 64  &  ReLU & Yes\\
				Conv'   &  96, 96, 3  &  ReLU & Yes\\
				\bottomrule
			\end{tabular} \quad \quad
			\begin{tabular}{l c c c c}
				\toprule
				Layer & Output & Activation & BN \\
				\cmidrule(lr){1-4}
				Input  & 96, 96, 3  & -  & -\\
				Conv  & 48, 48, 64  & Leaky ReLU  & No \\
				Conv  & 24, 24, 128  &  Leaky ReLU & Yes\\
				Conv  & 12, 12, 256  &  Leaky ReLU & Yes\\
				Conv  & 6, 6, 512  &  Leaky ReLU & Yes\\
				Flatten  & - &  - & - \\
				Dense  & 1  &  - & -\\
				\bottomrule
			\end{tabular}

		\end{center}
	\end{table}
	\begin{table}[h]
		\caption{The MDGAN encoder model architecture for the STL-10  experiments. The BN column indicates whether batch normalization is used after the layer or not.}
		\label{tab:encoder_architecture_stl}
		\begin{center}

			\begin{tabular}{l c c c c}
				\toprule
				Layer & Output & Activation & BN \\
				\cmidrule(lr){1-4}
				Input  & 96, 96, 3  & -  & -\\
				Conv   &  48, 48, 3  &  ReLU & Yes\\
				Conv   &   24, 24, 64  &  ReLU & Yes\\
				Conv   &  12, 12, 128  &  ReLU & Yes\\
				Conv   &   12, 12, 256  &  ReLU & Yes\\
				Flatten   &  -  &  - & -\\
				Dense  & 100  &  - & -\\
				\bottomrule
			\end{tabular}

		\end{center}
	\end{table}
	\FloatBarrier
	\subsection{ResNet Architectures}
	For CIFAR-10, we used the ResNet architecture from the appendix of~\cite{gulrajani2017improved} with minor changes as given in Table~\ref{tab:gen_dis_architecture_resnet_cifar10}. We used an initial learning rate of $5\mathrm{e}{-4}$ for CIFAR-10 and STL-10. For both datasets, the models are run for 200k iterations.
	For STL-10, we used a similar architecture that is given in Table~\ref{tab:gen_dis_architecture_resnet_stl10}.

	\begin{table}[h]
		\caption{The generator (top) and discriminator (bottom) ResNet architectures for the CIFAR-10  experiments.}
		\label{tab:gen_dis_architecture_resnet_cifar10}
		\begin{center}

			\begin{tabular}{l c c c c}
				\toprule
				Layer & Kernel Size & Resample & Output Shape \\
				\cmidrule(lr){1-4}
				Input  & -  & -  & $128$\\
				Dense  & -  & -  & $200 \cdot 4 \cdot 4$\\
				Reshape   & -  &  - & $200 \times 4 \times 4$\\
				ResBlock   &  $[3\times 3]\times 2$  &  up & $200\times8\times8$\\
				ResBlock  &  $[3\times 3]\times 2$  &  up & $200\times16\times16$\\
				ResBlock   &  $[3\times 3]\times 2$  &  up & $200\times32\times32$\\
				Conv, tanh   &  $3\times 3$  &  - & $3\times 32\times32$\\
				\bottomrule
			\end{tabular} \\ \vspace{0.5cm}
			\begin{tabular}{l c c c c}
				\toprule          
				Layer & Kernel Size & Resample & Output Shape \\
				\cmidrule(lr){1-4}
				ResBlock  & $[3\times 3]\times 2$  & Down  & $200\times16\times16$ \\
				ResBlock  & $[3\times 3]\times 2$  &  Down & $200\times8\times 8$\\
				ResBlock  & $[3\times 3]\times 2$  &  Down & $200\times4\times 4$\\
				ResBlock  & $[3\times 3]\times 2$  &  - & $200\times 4\times 4$\\
				ResBlock  & $[3\times 3]\times 2$  &  - & $200\times 4\times 4$\\
				ReLu, meanpool & - & - &200\\
				Dense & - & - &1\\
				\bottomrule
			\end{tabular}

		\end{center}
	\end{table}
	
	\begin{table}[h]
		\caption{The generator (top) and discriminator (bottom) ResNet architectures for the STL-10  experiments. For the experiment with full-sized 96x96 images, an extra upsampling block was added to the generator.}
		\label{tab:gen_dis_architecture_resnet_stl10}
		\begin{center}

			\begin{tabular}{l c c c c}
				\toprule
				Layer & Kernel Size & Resample & Output Shape \\
				\cmidrule(lr){1-4}
				Input  & -  & -  & 128\\
				Dense  & -  & -  & $200 \cdot 6 \cdot 6$\\
				Reshape   & -  &  - & $200 \times 6 \times 6$\\
				ResBlock   &  $[3\times 3]\times 2$  &  up & $200\times12\times12$\\
				ResBlock  &  $[3\times 3]\times 2$  &  up & $200\times24\times24$\\
				ResBlock   &  $[3\times 3]\times 2$  &  up & $200\times48\times48$\\
				Conv, tanh   &  $3\times 3$  &  - & $3\times 48\times48$\\
				\bottomrule
			\end{tabular} \\ \vspace{0.5cm}
			\begin{tabular}{l c c c c}
				\toprule          
				Layer & Kernel Size & Resample & Output Shape \\
				\cmidrule(lr){1-4}
				ResBlock  & $[3\times 3]\times 2$  & Down  & $200\times24\times24$ \\
				ResBlock  & $[3\times 3]\times 2$  &  Down & $200\times12\times 12$\\
				ResBlock  & $[3\times 3]\times 2$  &  Down & $200\times6\times 6$\\
				ResBlock  & $[3\times 3]\times 2$  &  - & $200\times 6\times 6$\\
				ResBlock  & $[3\times 3]\times 2$  &  - & $200\times 6\times 6$\\
				ReLu, meanpool & - & - &200\\
				Dense & - & - &1\\
				\bottomrule
			\end{tabular}

		\end{center}
	\end{table}

\end{document}